\documentclass{article}

\usepackage{amsmath, amssymb, graphicx, url, algorithm2e}

\usepackage{times}
\usepackage{graphicx}
\usepackage{psfrag}

 \newcommand{\suppx}[2][D]{\ensuremath{\mathit{supp}_{#1}({#2})}}
 \newcommand{\sxd}{\ensuremath{S_{X,D}}}
 \newcommand{\hsxd}{\ensuremath{H(S_{X,D})}}
 \newcommand{\cald}{\ensuremath{\mathfrak{I}(D)}}
 \newcommand{\calsxd}{\ensuremath{\mathfrak{I}(\sxd)}}
 \newcommand{\algesystem}{\ensuremath{(H,X)(D)}}
  \newcommand{\jacobian}{\ensuremath{J_X(D)}}
  \newcommand{\framework}[1][D]{\ensuremath{(H,X,#1)}}

\newcommand*{\Cdot}{{\scalebox{1.25}{$\,\cdot\,$}}}

\usepackage{color}
\usepackage{amsthm}

\newtheorem{claim}{Claim}
\newtheorem{question}{Question}
\newtheorem{problem}{Problem}
\newtheorem{example}{Example}
\newtheorem{remark}{Remark}
\newtheorem{theorem}{Theorem}

\newtheorem{lemma}{Lemma}
\newtheorem{corollary}[theorem]{Corollary}

\theoremstyle{plain}
\newtheorem{definition}[lemma]{Definition}

\newcommand\scalemath[2]{\scalebox{#1}{\mbox{\ensuremath{\displaystyle #2}}}}

\usepackage{multirow}
\usepackage{tabu}

\title{An Incidence Geometry approach to Dictionary Learning\thanks{This research was supported in part by the research grant NSF CCF-1117695 and a research gift from SolidWorks. A short version of this paper has appeared on Proceedings of the 26th Canadian Conference on Computational Geometry, 2014.}}

 \author{Meera Sitharam\thanks{Department of Computer \& Information Science \& Engineering, University of Florida}
 \and Mohamad Tarifi\thanks{Google Inc.}
 \and Menghan Wang\thanks{Department of Computer \& Information Science \& Engineering, University of Florida}}

\begin{document}

\maketitle

\begin{abstract}
We study the Dictionary Learning (aka Sparse Coding) problem 
of obtaining a sparse representation of data points, 
by learning \emph{dictionary vectors} upon which the data points
can be written as sparse linear combinations.
We view this problem from a geometry perspective as the spanning set of a subspace arrangement,
and focus on understanding the case when the underlying hypergraph of the subspace arrangement is specified. 
%
For this Fitted Dictionary Learning problem,
we completely characterize the combinatorics of the associated subspace arrangements (i.e.\ their underlying hypergraphs). 
Specifically, a combinatorial rigidity-type theorem is proven for a type of geometric incidence system.
The theorem characterizes the hypergraphs of subspace arrangements that generically yield 
(a) at least one dictionary
(b) a locally unique dictionary (i.e.\ at most a finite number of isolated dictionaries)
of the specified size.
We are unaware of prior application of combinatorial rigidity techniques 
in the setting of Dictionary Learning, or even in machine
learning.  
%
We also provide a systematic classification of problems related to Dictionary Learning 
together with various algorithms, their assumptions  and performance. 
\end{abstract}

\section{Introduction}

Dictionary Learning (aka Sparse Coding) is the problem of obtaining a sparse representation
of data points, by learning \emph{dictionary vectors} upon which the data points
can be written as sparse linear combinations.
\begin{problem}[Dictionary Learning] \label{prob:dictionary_learning}
A point set $X=[x_1 \ldots x_m]$ in $\mathbb{R}^d$ is said to be \emph{$s$-represented} by a \emph{dictionary} $D = [v_1 \ldots v_n]$ for a given \emph{sparsity} $s < d$,
if there exists $\Theta = [\theta_1 \ldots \theta_m]$ such that $x_i=D\theta_i$,  
with $\Arrowvert\theta_i\Arrowvert_0 \leq s$.
Given an $X$ known to be $s$-represented by an unknown \emph{dictionary} $D$ of size $|D|=n$, 
\emph{Dictionary Learning} is the problem of finding any  dictionary $\acute{D}$ 
satisfying the properties of $D$, i.e.\ $|\acute{D}| \le n$, 
and there  exists $\acute{\theta_i} $ such that $x_i=\acute{D} \acute{\theta_i}$  for all $x_i \in X$.
%
\end{problem}

The dictionary under consideration is usually \emph{overcomplete},  with $n > d$.
However we are interested in asymptotic performance with respect to all four variables $n,m,d,s$.
Typically, $m \gg n \gg d > s$. 
Both cases when $s$ is large relative to $d$ and when $s$ is small relative to $d$ are interesting.

The Dictionary Learning problem arises in various context(s) such as signal processing and machine learning.

\subsection{Review: Traditional and Statistical Approaches to Dictionary Learning}

A closely related problem to Dictionary Learning is
the Vector Selection (aka sparse recovery) problem, which finds a representation of input data in a known dictionary $D$. 

\begin{problem}[Vector Selection]
Given a dictionary $D \in \mathbb{R}^{d \times n}$  and an input data point
$x \in \mathbb{R}^d$, the \emph{Vector Selection problem} asks for $\theta \in \mathbb{R}^n$ 
such that $x = D\theta$ with $\Arrowvert \theta \Arrowvert_{0}$ minimized.
\end{problem} 

That is, $\theta$ is a sparsest support vector that represents $x$ as linear combinations of the
columns of $D$.

An optimization version of Dictionary Learning can be written as:
\[
\min_{D \in \mathbb{R}^{d\times n}}{ \max_{x_i} {\min \Arrowvert \theta_i \Arrowvert_0  : x_i = D\theta_i}}.  
\]
In practice, it is often relaxed to the Lagrangian 
$ \min \sum_{i=0}^m (\Arrowvert x_i - D\theta_i \Arrowvert_2  + \lambda \Arrowvert \theta_i \Arrowvert_1$).

Several traditional Dictionary Learning algorithms work by \emph{alternating minimization}, 
i.e.\ iterating the following two steps \cite{sprechmann2010dictionary,ramirez2010classification,olshausen1997sparse}:

1. Starting from an initial estimation of $D$, solving the Vector Selection problem for all data points $X$ to find a corresponding $\Theta$. This can be done using any vector selection algorithm, such as basis pursuit from \cite{chen1998atomic}. 

2. Given $\Theta$, updating the dictionary estimation by solving the optimization problem is now convex in $D$. 
%
For an 
overcomplete dictionary, 
the general Vector Selection problem is ill defined, 
as there can be multiple solutions for a data point $x$.
Overcoming this by framing the problem as a minimization problem is exceedingly difficult.
Indeed under generic assumptions, 
the Vector Selection problem has  been shown to be NP-hard
by reduction to the Exact Cover by 3-set problem \cite{rey1994adaptive}. 
%
%
One is then tempted to conclude that Dictionary Learning is also NP-hard. 
However, this cannot be directly deduced in general, 
since even though adding a witness $D$
turns the problem into an NP-hard problem, 
it is possible that the Dictionary Learning solves to produce
a different dictionary $\acute{D}$. 

On the other hand, if  $D$ satisfies the condition of being a \emph{frame}, 
i.e. for all $\theta$ such that $\Arrowvert \theta \Arrowvert_0 \leq s$, there exists a $\delta_s$ such that
$ (1-\delta_{s}) \leq \frac{ \Arrowvert D\theta \Arrowvert_2^2}{ \Arrowvert \theta \Arrowvert_2^2} \leq (1+\delta_{s})$,
it is guaranteed that the sparsest solution to 
the Vector Selection problem can be found via $L_1$ minimization \cite{donoho2006compressed,candes2006robust}.

\sloppy

One popular alternating minimization method is the Method of Optimal Dictionary (MOD) \cite{engan1999method},
which follows a two step iterative approach using a maximum likelihood formalism, 
and uses the pseudoinverse to compute $D$: $D^{(i+1)}=X\Theta^{(i)^T} (\Theta^n \Theta^{i^T})^{-1}.$
The MOD can be extended to Maximum A-Posteriori probability setting with different priors to take into account preferences in the recovered dictionary. 

\fussy

Similarly, $k$-SVD \cite{aharon2006svd} uses a two step iterative process, with a Truncated Singular Value Decomposition to update $D$.
This is done by taking every atom in $D$ and applying SVD to $X$ and $\Theta$ restricted to only the columns that have contribution from that atom. When $D$ is restricted to be of the form $D = [ B_1, B_2 \ldots B_L ]$ where $B_i$'s are orthonormal matrices, a more efficient pursuit algorithm
is obtained for the sparse coding stage using a block coordinate relaxation. 

Though alternating minimization methods work well in practice, there is no theoretical guarantee that the their results will converge to a true dictionary.
Several recent works 
give provable algorithms under stronger constraints on $X$ and $D$.
Spielman et.\ al \cite{spielman2013exact} give an $L_1$ minimization based approach which is provable to find the exact dictionary $D$, but requires $D$ to be a basis. 
Arora et.\ al \cite{arora2013new} and Agarwal et.\ al \cite{agarwal2013exact} 
independently give provable non-iterative algorithms for learning approximation of overcomplete dictionaries. 
Both of their methods are based on an overlapping clustering approach to find data points sharing a dictionary vector, 
and then estimate the dictionary vectors from the clusters via SVD. 
The approximate dictionary found using these two algorithms can be in turn used in iterative methods like k-SVD as the initial estimation of dictionary, leading to provable convergence rate \cite{agarwal2013learning}.
However, these overlapping clustering based methods require the dictionaries to have the \emph{pairwise incoherence} property
which is much stronger than the frame property.

\medskip

In this paper, 
we understand the Dictionary Learning problem from an intrinsically geometric point of view. 
Notice that 
each $x\in X$ lies in an $s$-dimensional subspace $\suppx{x}$,
which is the span of $s$ vectors $v \in D$ that form the \emph{support} of $x$. 
The resulting \emph{$s$-subspace arrangement} $\sxd = \{ (x,\suppx{x}): x\in X \}$ has an underlying 
labeled (multi)hypergraph
$\hsxd = (\cald, \calsxd)$, where $\cald$ denotes the index set of the dictionary $D$ and 
$\calsxd$ is the set of (multi)hyperedges over the indices $\cald$ 
corresponding to the labeled sets $(x,\suppx{x})$.
The word ``multi'' appears because if $\suppx{x_1} = \suppx{x_2}$ for data points 
$x_1,x_2 \in X$ with $x_1\not = x_2$, 
then that support set of dictionary vectors (resp.\ their indices)  is multiply represented in $\sxd$ (resp.\ $\calsxd$)
as labeled sets $(x_1,\suppx{x_1})$ and $(x_2,\suppx{x_2})$.

Note that there could be many dictionaries $D$ 
and for each $D$, many possible subspace arrangements $\sxd$ that are solutions to the 
Dictionary Learning problem.

\section{Contributions}

\textbf{Contribution to machine learning}:
In this paper, we focus on the version of  Dictionary Learning where the underlying hypergraph is specified. 
%
\begin{problem}[Fitted Dictionary Learning]  \label{prob:fitted_dictionary_learning}
Let $X$ be a given set of data points in $\mathbb{R}^d$. 
For an unknown  dictionary
$D = [v_1,\ldots, v_n]$ that $s$-represents $X$, 
we are given the hypergraph $\hsxd$ of the underlying subspace arrangement $\sxd$.
Find any dictionary $\acute{D}$ of size $|\acute{D}| \leq n$, 
that is consistent with the hypergraph $\hsxd$. 
\end{problem}

Our contributions to machine learning in this paper are as follows:
%
\begin{itemize}
\item As the \emph{main result}, we use combinatorial rigidity techniques to obtain a complete characterization of the 
hypergraphs $\hsxd$
that generically yield 
(a) at least one solution dictionary $D$, and
(b) a locally unique solution dictionary $D$ (i.e.\ at most a finite number of isolated solution dictionaries)
of the specified size (see Theorem \ref{thm:rigidity_condition}).
To the best of our knowledge, this paper pioneers the use of combinatorial rigidity for problems related to Dictionary Learning.
\item We are interested in minimizing $|D|$ for general $X$. However, 
as a \emph{corollary of the main result}, we obtain that
if the data points in $X$ are highly general, for example, 
 picked uniformly at random from the sphere $S^{d-1}$,
then when $s$ is fixed,
$|D| = \Omega(|X|)$ with probability 1 (see Corollary \ref{cor:dictionary_size_bound}).
\item As a \emph{corollary to our main result}, we obtain a Dictionary Learning algorithm for sufficiently general data $X$, i.e.\ 
requiring sufficiently large dictionary size $n$ (see Corollary \ref{corr:straight_forward_algo}). 
\item We provide a systematic classification of problems related to Dictionary Learning 
together with various approaches, assumptions required and performance
(see Section \ref{sec:review}).   
%
\end{itemize}

\smallskip

\noindent\textbf{Contribution to combinatorial rigidity}:

In this paper, we follow the combinatorial rigidity approaches \cite{asimow1978rigidity,white1987algebraic}
to give a complete combinatorial character for the geometric incidence system of the Fitted Dictionary Learning problem.
Specifically, 


\begin{itemize}
\item We formulate the Fitted Dictionary Learning problem 
as a nonlinear algebraic system $\algesystem$. 

\item We apply  classic method of Asimow and Roth \cite{asimow1978rigidity} to generically linearize the algebraic system $\algesystem$.

\item  We apply another well-known method of White and Whiteley \cite{white1987algebraic} to combinatorially characterize the rigidity of the underlying hypergraph $\hsxd$
and give so-called pure conditions that capture \emph{non-genericity}. 

\end{itemize}


To our best knowledge, 
the only known results with a similar flavor are \cite{haller2012body,lee-st.john2013combinatorics} which characterize the rigidity of Body-and-cad frameworks. 
However, these results are dedicated to specific frameworks in 3D instead of arbitrary dimension subspace arrangements and hypergraphs, 
and their formulation process start directly with the linearized Jacobian.

Note that although our results are stated for uniform  hypergraphs $\hsxd$
(i.e.\ each subspace in $\sxd$ has the same dimension), 
they can be easily generalized to non-uniform underlying hypergraphs.

 \section{Systematic classification of problems closely related to Dictionary Learning and previous approaches}
 \label{sec:review}

By imposing a systematic series of increasingly stringent constraints on the input, 
we classify previous approaches to Dictionary Learning 
as well as a whole set of independently interesting problems closely related to Dictionary Learning. 
A summary of the input conditions and results of these different types of Dictionary Learning approaches 
can be found in Table \ref{tab:classification}. 






A natural restriction of the general Dictionary Learning problem is the following.
We say that a set of data points $X$ \emph{lies on} a set $S$ of $s$-dimensional subspaces 
if  for all $x_i \in X$, there exists $S_i \in S$ such
that $x_i \in S_i$.
\begin{problem} [Subspace Arrangement Learning] \label{prob:subspace_arrangement_learning}
Let $X$ be a given set of data points that are known to lie on a set $S$ of
$s$-dimensional subspaces of $\mathbb{R}^d$, where $|S|$ is at most $k$.
(Optionally assume that the subspaces in $S$ have bases such that their
union is a frame).
\emph{Subspace arrangement learning} finds any subspace arrangement $\acute{S}$ of
$s$-dimensional subspaces of $\mathbb{R}^d$ satisfying these conditions, 
i.e.\ $|\acute{S}| \leq k$, $X$ lies on $\acute{S}$, (and optionally the union of the bases of $\acute{S}_i \in \acute{S}$ is a frame).
\end{problem}

There are several known algorithms for learning subspace arrangements. 
Random Sample Consensus (RANSAC) \cite{vidal2011subspace} is an approach to learning subspace arrangements that isolates, one subspace at a time, via random sampling.
When dealing with an arrangement of $k$ $s$-dimensional subspaces, for instance, the method samples $s+1$ points
which is the minimum number of points required to fit an $s$-dimensional subspace. The procedure then finds and discards
inliers by computing the residual to each data point relative to the subspace and selecting the points whose residual is below 
a certain threshold. The process is iterated until we have $k$ subspaces or all points are fitted.
RANSAC is robust to models corrupted with outliers. 
Another method called Generalized PCA (GPCA) \cite{vidal2005generalized} 
uses techniques from algebraic geometry   for subspace clustering, 
finding a union of $k$ subspaces by factoring a homogeneous polynomial of
degree $k$ that is fitted to the points $\{ x_1 \ldots x_m \}$.
Each factor of the polynomail represents the normal vector to a subspace. 
We note that GPCA can also determine $k$ if it is unknown.

The next problem is obtaining a minimally sized dictionary from a subspace arrangement.
%
\begin{problem} [\begin{small}Smallest Spanning Set for Subspace Arrangement\end{small}] \label{prob:smallest_spanning_set}
Let $S$ be a given
set of $s$-dimensional subspaces of $\mathbb{R}^d$ specified by giving their bases. 
Assume their intersections are known to be $s$-represented by a set $I$ of vectors with $\arrowvert I \arrowvert$ at most $n$. 
Find any set of vectors $\acute{I}$ that satisfies these conditions.
\end{problem}

The smallest spanning set is not necessarily unique in general, 
and is closely related to the intersection semilattice of subspace arrrangement \cite{bjorner1994subspace,goresky1988stratified}.
Furthermore, 
under the condition that the subspace arrangement comes from 
a frame dictionary, the smallest spanning set 
is the union of:
 (a) the smallest spanning set $I$ of the pairwise intersection of all the subspaces in $S$;
 (b) any points outside the pairwise intersections that, together with $I$, completely $s$-span the subspaces in $S$. 
 This directly leads to a recursive algorithm for the smallest spanning set problem. 

%
When $X$ contains sufficiently dense data to solve Problem \ref{prob:subspace_arrangement_learning}, 
Dictionary Learning reduces to problem \ref{prob:smallest_spanning_set}, 
i.e.\ we can use the following two-step procedure to solve the Dictionary Learning problem: 
\begin{itemize}
 \item Learn a Subspace Arrangement $S$ for $X$ (instance of Problem \ref{prob:smallest_spanning_set}). 
 \item Recover $D$ by finding the smallest Spanning Set of $S$ (instance of Problem \ref{prob:subspace_arrangement_learning}).
\end{itemize}

Note that it is not true that the decomposition strategy should always be applied for the same sparsity $s$. 
The decomposition starts out
with the minimum given value of $s$ and is reapplied with iteratively higher $s$ if a solution
has not be obtained.


A natural restriction of this two step problem learning subspace arrangement followed by spanning set is the following,
where the data set $X$ is given in \emph{support-equivalence} classes. 
For a given subspace $t$ in the subspace arrangement $\sxd$ (respectively hyperedge $h$ in the hypergraph's edge-set $\calsxd$), 
let $X_t = X_h  \subseteq X$ be the equivalence class of data points $x$ such that $\mathit{span}(\suppx{x}) = t.$ 
We call the data points $x$ in a same $X_h$ as \emph{support-equivalent}.
\begin{problem} [Dictionary Learning for Partitioned Data] \label{prob:support_equivalence_learning}
Given data $X$ partitioned into $X_i \subseteq X$, 
(1) What is the minimum size of $X$ and $X_i$'s guaranteeing 
that there exists a locally unique dictionary $D$ 
for a $s$-subspace arrangement $\sxd$ satisfying $|D| \le n$, 
and $X_i$ represents the support-equivalence classes of $X$ with respect to $D$?
(2) How to find such  a dictionary $D$? 
\end{problem}

With regard to the problem of minimizing $|D|$,  
very little is known for simple restrictions on $X$. 
For example the following question is open.
\begin{question} \label{question:general_position}
Given a general position assumption on $X$,
what is the best lower bound on $|D|$ for Dictionary Learning?
Conversely, are smaller dictionaries possible than indicated by 
Corollary~\ref{cor:dictionary_size_bound} (see Section~\ref{sec:result}) under such an assumption?
\end{question}

Question \ref{question:general_position} gives rise to the following pure combinatorics open question
closely related to the intersection semilattice of subspace arrrangement \cite{bjorner1994subspace,goresky1988stratified}.

\begin{question}
\label{question:lattice}
Given weights $w(S)\in \mathbb{N}$ assigned to size-$s$ subsets $S$ of $[n]$. 
For $T \subseteq [n]$ with $|T| \ne s$, 
\[
w(T) = 
\begin{cases}
0 & |T| < s\\
\displaystyle\sum_{S \subset T, |S| = s} w(S) & |T|> s 
\end{cases}
\]
Assume  additionally the following constraint holds:
for all subsets $T$ of $[n]$ with $s \le |T| \le d$, $w(T) \le |T|-1$.
Can one give a nontrivial upper  bound on $w([n])$?
\end{question}

The combinatorial characterization given by Theorem \ref{thm:rigidity_condition} leads to the following question
for general Dictionary Learning.
\begin{question}
\label{question:size_GDL}
What is the minimum size of a data set $X$ such that 
the Dictionary Learning for $X$ has a locally unique
solution dictionary $D$ of a given size? 
 What are the geometric characteristics of such an $X$? 
\end{question}

A summary of the input conditions and results of these different types of Dictionary Learning problems 
can be found in Table \ref{tab:classification}.

\renewcommand{\arraystretch}{1.2}
\begin{table*}[hptb]
\centering
\begin{footnotesize}
\resizebox{\textwidth}{!}{
\begin{tabular}{|p{0.12\textwidth}|p{0.1\textwidth}|p{0.1\textwidth}|p{0.125\textwidth}|p{0.23\textwidth}|p{0.165\textwidth}|p{0.18\textwidth}|}
\hline
\multirow{2}{*}{} & \multicolumn{3}{c|}{Traditional Dictionary Learning}  & 
\multirow{2}{\hsize}{Dictionary Learning via subspace arrangement and spanning set} & 
\multirow{2}{\hsize}{ Dictionary Learning for Segmented Data} & 
\multirow{2}{\hsize}{Fitted Dictionary Learning (this paper)}\\ \cline{2-4}
 & Alternating Minimization Approaches & Spielman et. al \cite{spielman2013exact} & Arora et. al \cite{arora2013new}, Agarwal et. al \cite{agarwal2013exact} &  	&  &  \\ \hline
 

 \multirow{2}{\hsize}{Input and Conditions}  & \multirow{2}{\hsize}{$D$ satisfies frame property} & \multicolumn{2}{p{0.24\hsize}|}{$X$ generated from hidden dictionary $D$ and certain distribution of $\Theta$} & \multirow{2}{\hsize}{$X$ with promise that each subspace / dictionary support set is shared by sufficiently many of the data points in $X$} & \multirow{2}{\hsize}{Partitioned / segmented Data $X$} & \multirow{2}{\hsize}{Generic data points $X$ (satisfying pure condition) with underlying hypergraph specified}   \\ \cline{3-4}
& &  $D$ is a basis &  $D$~is~pairwise incoherent & & & \\ \hline

Minimum $m$ guaranteeing existence of a locally unique dictionary of a given size $n$ & Question~\ref{question:size_GDL} & $O(n \log n)$ & $O(n^2 \log^2 n)$ & Minimum number of points to guarantee a unique subspace
arrangement that will give a spanning set of size $n$ & Problem~\ref{prob:support_equivalence_learning} & $\dfrac{d-s}{d-1} n$ (Theorem~\ref{thm:rigidity_condition}); Unknown for general position data (Question~\ref{question:general_position}) \\ \hline

Dictionary Learning algorithms & MOD, k-SVD, etc. & Algorithm from \cite{spielman2013exact} & Algorithms from \cite{arora2013new,agarwal2013exact} & 
Subspace Arrangement Learning Algorithms (Problem~\ref{prob:subspace_arrangement_learning}) and Spanning Set Finding ( Problem~\ref{prob:smallest_spanning_set}) & Problem~\ref{prob:support_equivalence_learning} and Spanning Set Finding ( Problem~\ref{prob:smallest_spanning_set}) & Straightforward algorithm (Corollary~\ref{corr:straight_forward_algo}) 
 \\ \hline

Minimum $m$ guaranteeing efficient dictionary learning & Unknown & \multicolumn{2}{p{0.2\hsize}	|}{$O(n^2 \log^2 n)$} & Unknown & Unknown  &Unknown \\ \hline
Illustrative example & \multicolumn{3}{p{0.3\hsize}|}{

}  & 
(a) \includegraphics[width=\hsize]{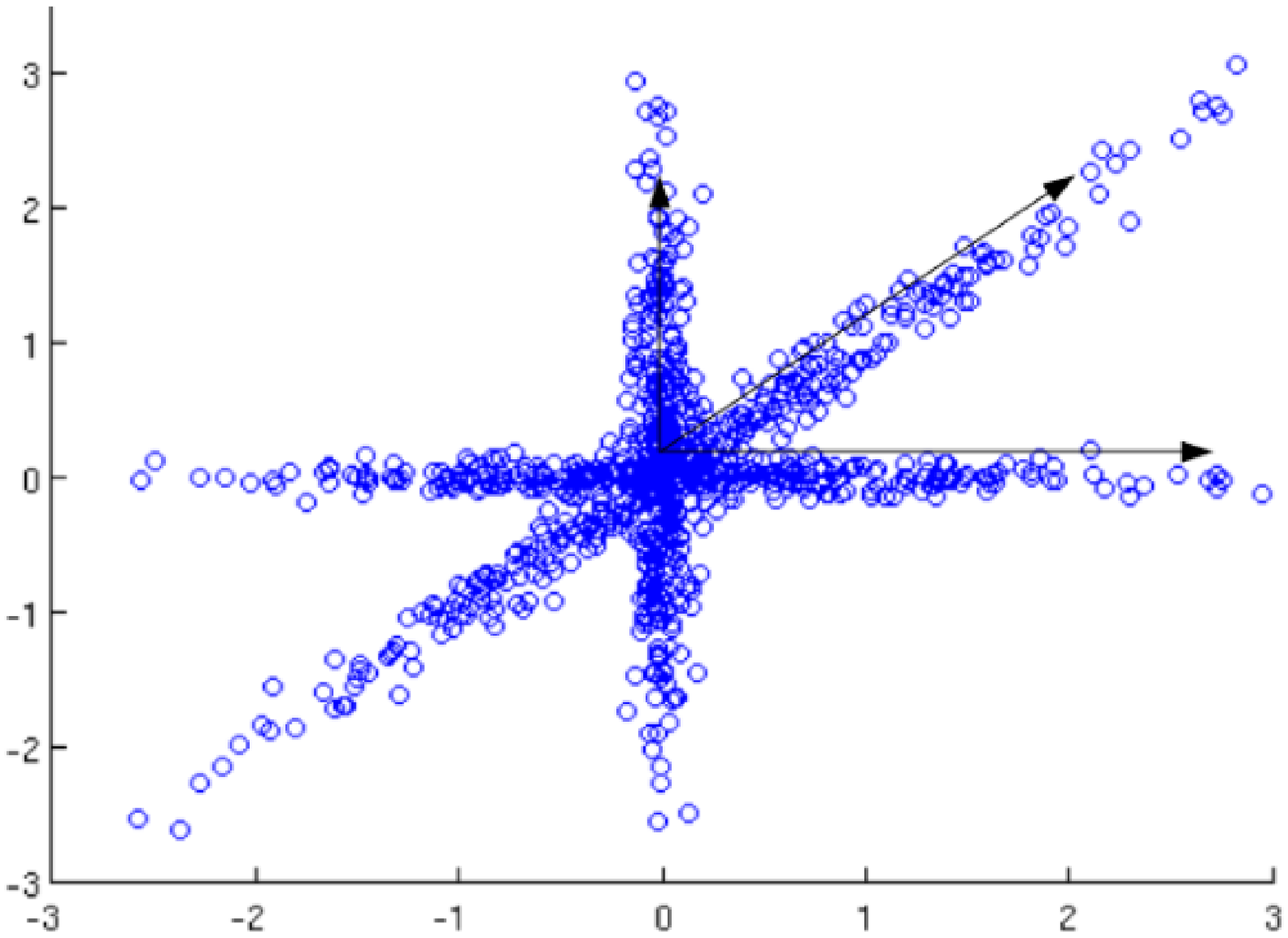}
 & 
 (b) \includegraphics[width=\hsize]{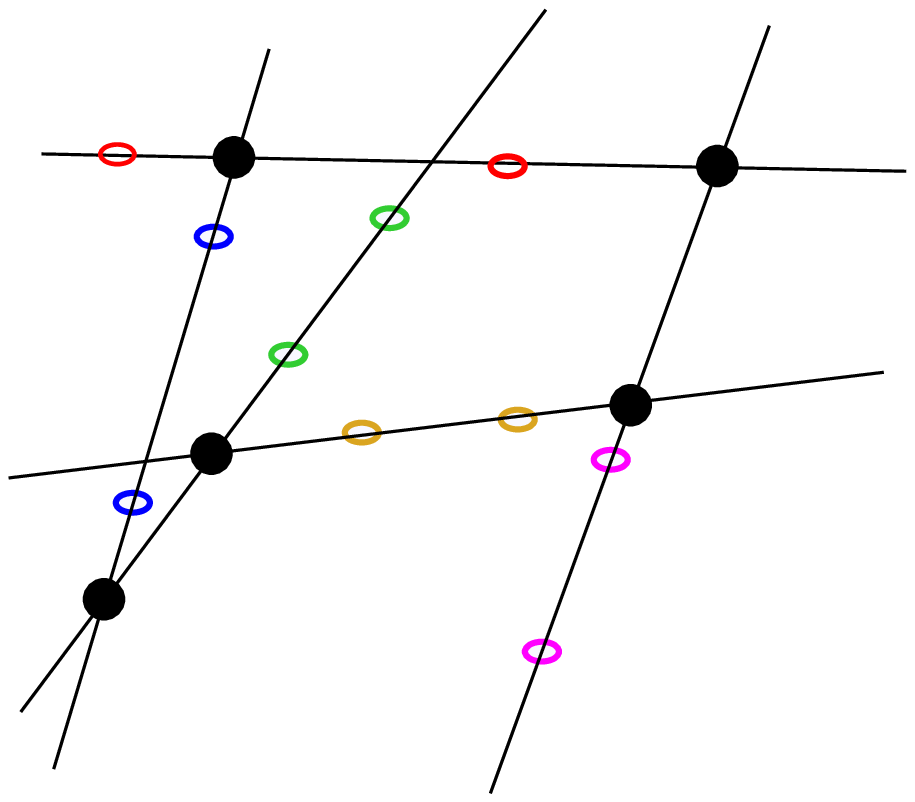}
 & 
 (c) \includegraphics[width=\hsize]{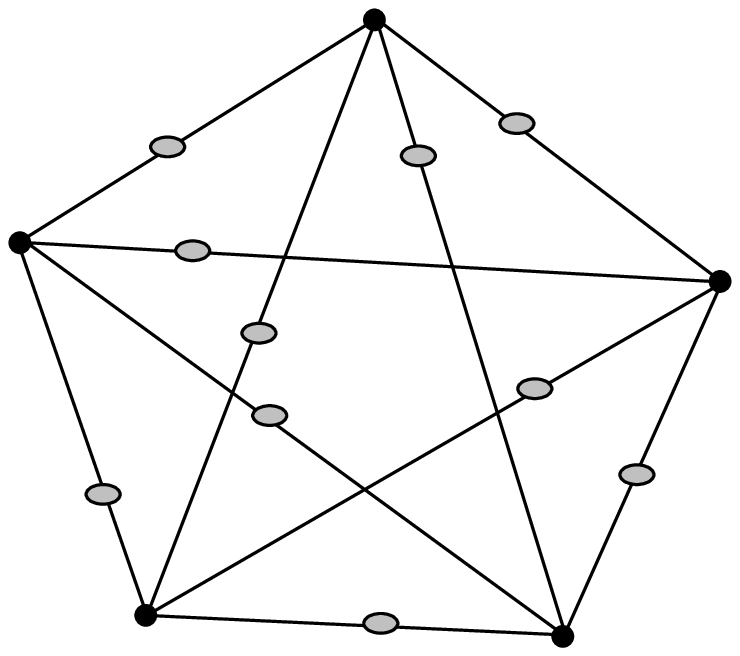}

 \\ \hline
\end{tabular}}

  \end{footnotesize}
\caption{Classification of Problems}
\label{tab:classification}
\end{table*}

\section{Main Result: Combinatorial Rigidity Characterization for Dictionary Learning}
\label{sec:result}

In this section, we present 
the main result of the paper, i.e.\ a complete solution to the problem of finding a dictionary $D$ for data $X$, when 
the hypergraph $\hsxd$ of the underlying subspace arrangement is specified.
Additionally we give a (combinatorial)
characterization of the hypergraphs $H$ such that the existence and
local uniqueness of a dictionary $D$ is guaranteed for generic $X$ satisfying  $H = \hsxd$.

Since the magnitudes of the vectors in $X$ or $D$ are uninteresting, we
treat the data and dictionary points in the projective $(d-1)$-space and
use the same notation to refer to both original $d$-dimensional and
projective $(d-1)$-dimensional versions when the meaning is clear from the
context.
We rephrase the Fitted Dictionary Learning problem as the following 			
Pinned Subspace-Incidence problem for the convenience of applying 			
machinery from incidence geometry. 

\begin{problem} [Pinned Subspace-Incidence Problem] \label{prob:pinned_subspace}
Let $X$ be a given set of $m$ points (pins) in $\mathbb{P}^{d-1}(\mathbb{R})$. For every pin $x\in X$, we 
are also given the hyperedge $\suppx{x}$, 
i.e, an index subset of an unknown set of points $D = \{v_1,\ldots,v_n\}$, 
such that  $x_i$ lies on the subspace spanned by $\suppx{x}$. 
Find any such set D 
that satisfies the given subspace incidences. 
\end{problem}

In the following, we give combinatorial conditions that characterize the class of inputs that recover a finite number of solutions $D$.

\subsection{Algebraic Representation, Rigidity and Linearization}

We represent the Pinned Subspace-Incidence problem in the tradition of geometric constrain solving \cite{bruderlin1998geometric,sitharam2005combinatorial},
and view it as finding the common solutions of 
an algebraic system \eqref{eq:system}: $\algesystem = 0$  (finding a real algebraic variety).
We then use the approach taken by the traditional  rigidity theory \cite{asimow1978rigidity,graver1993combinatorial} 
for characterizing generic properties of these solutions. 
The details are given in  Section \ref{sec:algebraic_representation} and \ref{sec:rigidity}.

\sloppy
We define a \emph{pinned subspace-incidence framework} 
of an underlying hypergraph $\hsxd=(\cald,\calsxd)$
to be the triple $\framework$,
where 
$X: \{x_1,\ldots, x_m\} \subseteq \mathbb{R}^{d-1}\rightarrow \calsxd$ 
is an assignment of a given set of pins $x_k$ to edges
$ X(x_k) = \suppx{x_k} \in \calsxd$, 
and $D: \cald \rightarrow \mathbb{R}^{d-1}$ is an embedding of each vertex $j$ 
into 
a point $v_j \in \mathbb{R}^{d-1}$, such that
each pin $x_k$ lies on the subspace spanned by $\{v^k_1,v^k_2, \ldots, v^k_s\}$.
Note: when the context is clear, we use $X$ to denote both the set of points 
$\{ x_1 , \ldots , x_m \}$ , as well as the above assignment of these points to
edges of $H$.
Two frameworks $(H_1,X_1,D_1)$ and $(H_2,X_2,D_2)$ are \textit{equivalent} if
$H_1=H_2$ and $X_1 =X_2$, i.e.\ they satisfy the same algebraic equations for the same labeled hypergraph and ordered set of pins. 
They are \emph{congruent} if they are equivalent and $D_1=D_2$.

\fussy
The pinned subspace-incidence system is \emph{minimally rigid} if 
it is both \emph{independent}, i.e.\  none of the algebraic constraints is in the ideal generated by the others, 
which generically implies the existence of a (possibly complex) solution $D$ to the system $\algesystem$,
and \emph{rigid}, i.e. there exist at most finitely many (real or complex) solutions.
%
Rigidity is often defined (slightly differently) for individual frameworks. 
A framework $\framework$ is \emph{rigid} (i.e. \emph{locally unique}) if there is a neighborhood $N(D)$ of $D$, 
such that any framework $\framework[D']$ equivalent to $\framework$ with $D' \in N(D)$ is also congruent to $\framework$. 
A rigid framework $\framework$ is \emph{minimally rigid} if it becomes flexible after removing any pin.

We are interested in characterizing \emph{minimal rigidity} 
of the pinned subspace-incidence system and framework.
However, checking independence relative to the ideal generated by the variety is computationally hard and best known algorithms,
 such as computing Gr\"{o}bner basis, are exponential in time and space \cite{mittmann2007grobner}.
However, the algebraic system can be linearized at \emph{generic} or \emph{regular} (non-singular) points
(formally defined in Section \ref{sec:genericity}).
%
%
Adapting \cite{asimow1978rigidity}, we show in Section \ref{sec:rigidity} that 
rigidity and independence (based on nonlinear polynomials) of pinned subspace-incidence systems 
are generically properties of the underlying hypergraph $\hsxd$.
Specifically, Lemma \ref{lem:generic} shows that
rigidity a pinned subspace-incidence system
is equivalent to existence of a full rank  \emph{rigidity matrix}, 
obtained by taking the Jacobian $\jacobian$ of the algebraic system $\algesystem$ at a regular point.

\subsection{Statement of results}

We study the rigidity matrix to
obtain the following combinatorial characterization of 
(a) sparsity / independence, i.e.\ existence of a dictionary,  
and (b) rigidity, i.e.\ the solution set being locally unique / finite,
for a pinned subspace-incidence framework.

\begin{theorem}[Main Theorem] \label{thm:rigidity_condition}
A pinned subspace-incidence framework is generically minimally rigid
if and only if the underlying hypergraph $\hsxd = (\cald,\calsxd)$ satisfies 
$(d-s)| \calsxd | = (d-1)|\cald|$ (i.e. $(d-s)|X| = (d-1)|D|$), 
and $(d-s)|E'| \le (d-1)|V'|$ for every vertex induced subgraph $H'=(V',E')$.
The latter condition alone ensures the independence
of the framework.
\end{theorem}

We relate the Fitted Dictionary Learning problem to the general Dictionary Learning problem. 
The following is a useful corollary to the main theorem, 
which gives the lower bound of dictionary size for generic data points.
\begin{corollary}[Dictionary size lower bound for generic data]\label{cor:dictionary_size_bound_generic}
Given a set of $m$ points $X=\{x_1,..,x_m\}$ in $\mathbb{R}^d$, generically there
is a dictionary $D$ of size $n$ that $s$-represents $X$ 
only if $(d-s)m \le (d-1)n$. 
Conversely, if $(d-s)m = (d-1)n$ and the supports of $x_i$ (the nonzero entries of
the $\theta_i$'s) are known to form a $(d-1,0)$-tight hypergraph $H$, then generically, there
is at least one and at most finitely many such dictionaries.
\end{corollary}

%

Quantifying the term ``generically'' in Corollary \ref{cor:dictionary_size_bound_generic}
yields Corollaries \ref{cor:dictionary_size_bound} and \ref{corr:straight_forward_algo} below.  
%
\begin{corollary}[Lower bound for highly general data]\label{cor:dictionary_size_bound}
Given a set of $m$ points $X=\{x_1,..,x_m\}$ picked uniformly at random from  the sphere $S^{d-1}$, 
a dictionary $D$ that $s$-represents $X$ has size at least $(\frac{d-s}{d-1})m$ with probability 1.
In other words, 
$|D| = \Omega(X)$ if $s$ and $d$ are constants.
\end{corollary}

Proofs of Theorem \ref{thm:rigidity_condition} and Corollary \ref{cor:dictionary_size_bound_generic}, \ref{cor:dictionary_size_bound} are given in Section \ref{sec:proof}.

Corollary \ref{corr:straight_forward_algo} provides an algorithm to construct a dictionary for highly
general data points. The algorithm is elaborated in Section \ref{sec:algorithm}.

\begin{corollary}[Straightforward Dictionary Learning Algorithm]\label{corr:straight_forward_algo}
Given a set of $m$ points  $X=[x_1 \ldots x_m]$ picked uniformly at random from  the sphere $S^{d-1}$, 
we have a straightforward algorithm to construct a dictionary $D=[ v_1 \ldots v_n ]$ that $s$-represents $X$,
where $n = \left( \dfrac{d-s}{d-1} \right) m$.
The time complexity of the algorithm is $O(m)$ when we treat $d$ and $s$ as constants. 
\end{corollary}

\section{Proof of Results}

In this section,  we provide details and proof for the results given in the last section. 

In the following,  
we denote a minor of a matrix $A$ using the notation $A[R,C]$, 
where $R$ and $C$ are index sets of the rows and columns contained in the minor, respectively.
In addition, $A[R,\Cdot]$ represents the minor containing all columns and row set $R$, 
and $A[\Cdot,C]$ represents the minor containing all rows and column set $C$.
\subsection{Algebraic Representation}
\label{sec:algebraic_representation}

In this section, 
we provide the details in deriving the algebraic system of equations $\algesystem = 0$ \eqref{eq:system}
to represent our problem,
in the tradition of geometric constrain solving \cite{bruderlin1998geometric,sitharam2005combinatorial}.

Consider a pin $x_k$ on the subspace spanned by 
the point set  $S^k = \{v_1^k, v_2^k, \ldots, v_s^k\}$. 
Using homogeneous coordinates, 
we can write this incidence constraint  by letting 
all the $s \times s$ minors of the $s \times (d-1) $ matrix
\[
E^k=
\left[
\begin{array}{c}
v^k_1 - x_{k}	\\
 v^k_2 -x_{k} \\
\vdots \\
 v^k_s - x_{k} 
\end{array} \right]
\]
be zero, where $v^k_i = [\begin{array}{cccc}v^k_{i,1} & v^k_{i,2} & \ldots & v^k_{i,d-1}\end{array}] $ and
$x_{k} = [ \begin{array}{cccc}x_{k,1} & x_{k,2} & \ldots & x_{k,d-1} \end{array}]$.
So each incidence can be written as $d-1 \choose s$ equations: 
\begin{equation} \label{eq:constraint}
\det\left(E^k[\Cdot, C(t)]\right) = 0, \qquad 1 \le t \le {d-1 \choose s}
\end{equation} 
where $C(t)$ enumerates all the $s$-subsets of columns of $E^k$. 
Note that only $d-s$ of these $d-1 \choose s$ equations are independent,
as the span of $S^k$
is a $s$-dimensional subspace in a $d$-dimensional space,
which only has $s(d-s)$ degrees of freedom.

Given the hypergraph $H=\hsxd$ of the underlying subspace arrangement,
the pinned subspace-incidence  problem   now reduces to solving a system of $m{d-1 \choose s}$ equations
(or, equivalently, $m(d-s)$ independent equations), 
each of form \eqref{eq:constraint}. 
The system of equations sets a multivariate function  $\algesystem$ to $0$:
\begin{equation} \label{eq:system}
\algesystem = \begin{cases}
\qquad\ldots\\
\det\left(E^k[\Cdot, C(t)]\right) & \qquad = \quad 0 \\ 
\qquad\ldots
\end{cases}
\end{equation}
When viewing $X$ as a fixed parameter, 
$\algesystem$ is a vector valued function from $\mathbb{R}^{n(d-1)}$ to $\mathbb{R}^{m{d-1 \choose s}}$ parameterized by $X$.

Without any pins, the points in $D$ have in total $n(d-1)$ degrees of freedom. 
In general, putting $r$ pins on an $s$-dimensional subspace of $d$-dimensional space 
gives an $(s-r)$-dimensional subspace of a $(d-r)$-dimensional space, 
which has $(s-r)((d-r)-(s-r)) = (s-r)(d-s)$ degrees of freedom left.
So every pin potentially removes $(d-s)$ degrees of freedom.

%

As introduced in Section \ref{sec:result},
the pinned subspace-incidence system $\algesystem$ is \emph{independent} 
if none of the algebraic constraints is in the ideal generated by the others. 
Generally, independence implies the existence of a solution $D$ to the system $\algesystem$, where X is fixed. 
The system is \emph{rigid}  if there exist at most finitely many (real or complex) solutions.  
The system is \emph{minimally rigid} if 
it is both rigid and independent.
The system is \emph{globally rigid} if there exists at most one solution.
A pinned subspace-incidence framework $\framework$ is \emph{rigid} (i.e. \emph{locally unique}) if there is a neighborhood $N(D)$ of $D$, 
such that any framework $\framework[D']$ equivalent to $\framework$ with $D' \in N(D)$ is also congruent to $\framework$. 
A rigid framework $\framework$ is \emph{minimally rigid} if it becomes flexible after removing any pin.
A framework $\framework$ is \emph{globally rigid} (i.e. \emph{globally unique}) if any framework equivalent to $\framework$ is also congruent to $\framework$. 

\subsubsection{Genericity}
\label{sec:genericity}

We are interested in characterizing minimal rigidity 
of the pinned subspace-incidence system and framework.
However, checking independence relative to the ideal generated by the variety is computationally hard and the best known algorithms
 are exponential in time and space \cite{mittmann2007grobner}. 
However, the algebraic system can be linearized at \emph{generic} or \emph{regular} 
(non-singular) points
whereby independence and rigidity of the algebraic pinned subspace-incidence system $\algesystem$ 
reduces to linear independence and maximal rank at \emph{generic} frameworks. 

In algebraic geometry, a property being generic intuitively means that the property holds on the open dense complement of an (real) algebraic variety. Formally, 

\begin{definition}
\label{def:generic_framework}
A framework $\framework$ is generic w.r.t.\ a property $Q$ if and only if there exists a neighborhood $N(D)$ such that for all frameworks $\framework[D']$ with $D' \in N(D)$,
$\framework[D']$ satisfies $Q$ if and only if $\framework$ satisfies $Q$.
\end{definition}

Furthermore we can define generic properties of the hypergraph.

\begin{definition}
\label{def:generic_property}
A property $Q$ of frameworks is generic (i.e, becomes a property of the
hypergraph alone) if for all graphs $H$, either all generic (w.r.t.\ $Q$) frameworks
$\framework$ satisfies $Q$, or all generic (w.r.t.\ $Q$) frameworks $\framework$ do not satisfy
$Q$.
\end{definition}

A framework $\framework$ is \emph{generic} for property $Q$ if  an algebraic variety $V_Q$ specific to $Q$ is
avoided by the given framework $\framework$.
Often, for convenience in relating $Q$ to other properties, a more
restrictive notion of genericity is used than stipulated by Definition \ref{def:generic_framework} or \ref{def:generic_property},
i.e.\ another variety $\acute{V}_Q$ is  chosen so that
$V_Q \subseteq \acute{V}_Q$, as in Lemma \ref{lem:generic}. 
Ideally, the variety $\acute{V}_Q$ corresponding to the chosen notion of genericity
should be as tight as possible for the property Q (necessary and
sufficient for Definition \ref{def:generic_framework} and \ref{def:generic_property}), 
and should be explicitly defined, or at
least easily testable for a given framework.

Once an appropriate notion of genericity is defined, 
we can treat $Q$  as a property of a hypergraph.
The primary activity of the area of combinatorial rigidity is to
  give purely combinatorial characterizations of such generic
properties $Q$.
In the process of drawing such combinatorial characterizations,
the notion of genericity  may have to be further restricted, i.e.\ the
variety $\acute{V}_Q$ is further expanded by so-called pure conditions that are
necessary for the combinatorial characterization to go through 
(we will see this   in the proof of Theorem \ref{thm:rigidity_condition}).

\subsection{Linearization as Rigidity Matrix and its Generic Combinatorics}
\label{sec:rigidity}




%

Next we follow the approach taken by  traditional combinatorial rigidity theory \cite{asimow1978rigidity,graver1993combinatorial} 
to show that rigidity and independence (based on nonlinear polynomials) of pinned subspace-incidence systems 
are generically properties of the underlying hypergraph $\hsxd$,
and can furthermore be captured by linear conditions in an infinitesimal setting.
Specifically, Lemma~\ref{lem:generic} shows that
rigidity a pinned subspace-incidence system
is equivalent to existence of a full rank  \emph{rigidity matrix}, 
obtained by taking the Jacobian of the algebraic system $\algesystem$ at a regular point.

A \emph{rigidity matrix} of a framework $\framework$ is a matrix  whose kernel is the infinitesimal motions (flexes) of $\framework$. 
A framework is \emph{infinitesimally independent} if the rows of the rigidity matrix are independent.
A framework is \emph{infinitesimally rigid} if the space of infinitesimal motion is trivial, 
i.e.\ the rigidity matrix has full rank. 
A framework is \emph{infinitesimally minimally rigid} if it is both infinitesimally independent and rigid.

To define a rigidity matrix for a pinned subspace-incidence framework $\framework$, 
we take the Jacobian $\jacobian$ of the algebraic system $\algesystem$, 
by taking partial derivatives w.r.t.\ the coordinates of $v_i$'s. 
In the Jacobian, each vertex $v_i$ has $d-1$  corresponding columns, 
and each pin $x_k$ has 
$d-1 \choose s$ corresponding rows, \sloppy
where each equation $\det\left(E^k[\Cdot, C(t)]\right)=0$ \eqref{eq:constraint} gives the following row 
(where  $x_k$ lies on the subspace spanned by  $S^k = \{v^k_1, v^k_2, \ldots, v^k_s \}$):

\fussy
\begin{align}
\label{eq:row}
[0,&\ldots,0,0, \frac{\partial \det\left(E^k[\Cdot, C(t)]\right)}{\partial v^k_{1,1}}, \frac{\partial \det\left(E^k[\Cdot, C(t)]\right)}{\partial v^k_{1,2}}, \ldots,\frac{\partial \det\left(E^k[\Cdot, C(t)]\right)}{\partial v^k_{1,d-1}}, 0,0, \notag \\ 
&\ldots,0,0,\frac{\partial \det\left(E^k[\Cdot, C(t)]\right)}{\partial v^k_{2,1}}, \frac{\partial \det\left(E^k[\Cdot, C(t)]\right)}{\partial v^k_{2,2}},  \ldots, \frac{\partial \det\left(E^k[\Cdot, C(t)]\right)}{\partial v^k_{2,d-1}}, 0,0,\ldots \notag \\ 
&\ldots \ldots \notag \\ 
&\ldots, 0,0, \frac{\partial \det\left(E^k[\Cdot, C(t)]\right)}{\partial v^k_{s,1}}, \frac{\partial \det\left(E^k[\Cdot, C(t)]\right)}{\partial v^k_{s,2}},  \ldots, \frac{\partial \det\left(E^k[\Cdot, C(t)]\right)}{\partial v^k_{s,d-1}}, 0,\ldots,0]
\end{align}

For $j \in C(t)$, let $V^{k}_{i,j}(t)$ be the $(s-1)$-dimensional oriented volume, of the 
$(s-1)$-simplex formed by the vertices  
$(S^k \setminus \{v^k_i\})$ together with  $x_k$, 
projected on the coordinates $C(t,j) = C(t)\setminus \{j\}$. 
We define the following function  $\delta$ which adds the appropriate sign to $V^k_{i,j}(t)$:
\[
\delta\, V^k_{i,j}(t) = 
\begin{cases}
 (-1)^{q} \, V^k_{i,j}(t) & \text{if } j \in C(t), \text{where } q \text{ is the index of column } j \text{ in } C(t)\\
 0 & \text{if } j \notin C(t)
\end{cases}
\]
Now \eqref{eq:row} is equal to
\begin{align*}
r^k(t) = 
[0,&\ldots,0,0, \delta\, V^k_{1,1}(t), \delta\, V^k_{1,2}(t), \ldots,\delta\, V^k_{1,d-1}(t), 0,0,\\
&\ldots \ldots \\
&\ldots, 0,0, \delta\, V^k_{s,1}(t), \delta\, V^k_{s,2}(t),  \ldots, \delta\, V^k_{s,d-1}(t), 0,0,\ldots,0]
\end{align*}
Each vertex $v^k_{i}$ has the entries $\delta\,V^k_{i,1}(t), \delta\,V^k_{i,2}(t), \ldots, \delta\,V^k_{i,d-1}(t)$ in its $d-1$ columns, among which $s$ entries are generically non-zero.



Notice that for every pair of vertices $v^k_i$ and $v^k_{i'}$, 
the projected volumes on different coordinates have the same ratio:
$\displaystyle\frac{V^k_{i,j_2}(t)}{ V^k_{i,j_1}(t)} = \frac{ V^k_{i',j_2}(t)}{ V^k_{i',j_1}(t)}$ for all $1 \le j_1, j_2 \le d-1, j_1 \in C(t), j_2 \in C(t)$.
So we can divide each row $r^k(t)$ by  
$\sum_{i=1}^s V^k_{i,{j^*}}(t)$,
where $j^*$ is any index in $C(t)$,
and simplify $r^k(t)$  to 
\begin{small}
\begin{align}
[ 0, &\ldots, 0,0, \delta\,b^k_{C(t,1)} a^k_1, \delta\,b^k_{C(t,2)} a^k_1 \ldots, \delta\,b^k_{C(t,d-1)} a^k_1, 0, 0,  \notag \\
&\ldots,   0, 0,  \delta\,b^k_{C(t,1)}a^k_2, \delta\,b^k_{C(t,2)}a^k_2, \ldots, \delta\,b^k_{C(t,d-1)} a^k_2, 0, 0,  \notag\\
 &\ldots \ldots, \notag \\
&\ldots,0, 0,  \delta\,b^k_{C(t,1)} a^k_s, \delta\,b^k_{C(t,2)} a^k_s,\ldots, \delta\,b^k_{C(t,d-1)} a^k_s,  0, 0, \ldots,0 ]  \label{eq:row_pattern}
\end{align}
\end{small}
where  $\sum_{1 \le i \le s} a^k_i = 1$.

\begin{figure}[hbtp]
\begin{center}
\includegraphics[width=.25\linewidth]{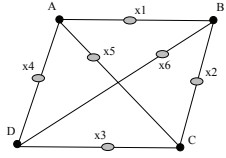}
\end{center}
\caption{An pinned subspace-incidence framework  of $6$ pins and $4$ vertices, with $d=4,s=2$.}
\label{fig:tetrahedron}
\end{figure}

\begin{example} \label{example:k4}
Figure \ref{fig:tetrahedron} shows a pinned subspace-incidence framework with $d = 4, s=2$.
If we denote
$\alpha_{i,j} = A_j - x_{i,j}, \beta_{i,j} = B_j - x_{i,j}$, 
the edge $AB$ will have the following three rows in the Jacobian: 
\[
\left[
\scalemath{0.75}{
\begin{array}{cccccccccccc}
\beta_{1,2}& -\beta_{1,1} & 0 & -\alpha_{1,2} & \alpha_{1,1} & 0 &0&0&0&0&0&0\\
\beta_{1,3}& 0 & -\beta_{1,1} & -\alpha_{1,3}& 0 & \alpha_{1,1} &0&0&0&0&0&0\\
0 & \beta_{1,3} & -\beta_{1,2} & 0&-\alpha_{1,3}& \alpha_{1,2} &0&0&0&0&0&0
\end{array} }
\right]
\]
and the corresponding rows in the simplified Jacobian has the following form
\[
\left[
\scalemath{0.75}{
\begin{array}{cccccccccccc}
 b_2 a_1 &  -b_1 a_1 & 0 &  -b_2 a_2 &  b_1 a_2 & 0 &0&0&0&0&0&0\\
 b_3 a_1 & 0 &  -b_1 a_1 &  -b_3 a_2 & 0 &  b_1 a_2 &0&0&0&0&0&0\\
 0 & b_3 a_1 &  -b_2 a_1 &  0& -b_3 a_2 &   b_2 a_2 &0&0&0&0&0&0
\end{array} }
\right]
\] 
\end{example}

For a pinned subspace-incidence framework $\framework$, 
we define the \emph{symmetric rigidity matrix $M$}
to be the simplified Jacobian matrix obtained above, of size $m{d-1 \choose s}$ by $n(d-1)$,
where each row has the form \eqref{eq:row_pattern}.
Notice that in $M$ each hyperedge has $d-1 \choose s$ rows, 
where any $d-s$ of them are independent, and spans all other rows. 
If we choose  $d-s$ rows per hyperedge in $M$, 
the obtained matrix $\hat{M}$ is a rigidity matrix of size $m(d-s)$ by $n(d-1)$.
The framework is infinitesimally rigid if and only if there is an $\hat{M}$ with full rank.
Note that the rank of a generic matrix $\hat{M}$
is at least as large as the rank of any specific realization $\hat{M}\framework$.

\begin{remark}
There are several correct ways to write the rigidity matrix of a framework, 
depending on what one considers as the primary indeterminates (points, subspaces, or both),
i.e.\ whether one chooses to work in primal or dual space. 
We pick points for columns 
for the simplicity of the row pattern.
\end{remark}

Defining generic as non-singular, 
we show that for a generic framework $\framework$,
infinitesimal rigidity is equivalent to generic rigidity. 
%
\sloppy
\begin{lemma}
\label{lem:generic}
If $D$ and $X$ are regular / non-singular with respect to the system $\algesystem$, 
then generic infinitesimal rigidity of the framework $\framework$ is equivalent to generic rigidity.
\end{lemma}

\fussy
{\em Proof Sketch}. 
First we show that if a framework is regular, infinitesimal rigidity implies rigidity.
Consider the polynomial system $\algesystem$ of equations. 
The Implicit Function Theorem states that 
there exists a function $g$, such that $D=g(X)$ on some open interval, if and only if the Jacobian $\jacobian$ of $\algesystem$ with respect to $D$ has full rank. 
Therefore, if the framework is infinitesimally rigid, then the solutions to the algebraic system are isolated points (otherwise $g$ could not be explicit). 
Since the algebraic system contains finitely many components, there are only finitely many such solution 
and each solution is a $0$ dimensional point. This implies that the total number of solutions is finite,
 which is the definition of rigidity.

To show that generic rigidity implies generic infinitesimal rigidity, we take the contrapositive: if a generic framework is not infinitesimally rigid,
we show that there is a finite flex. 
Let $\hat{M}$ be the $m(d-s)$ by $n(d-1)$ rigidity matrix obtained from the Jacobian $\jacobian$ which has the maximum rank. 
If $\framework$ is not infinitesimally rigid,
then the rank $r$ of $\hat{M}$ is less than $n(d-1)$. 
Let $E^*$ be a set of edges in $H$ such that $|E^*|=r$ and the corresponding rows in the Jacobian $\jacobian$ are all independent.
In $\hat{M}[E^*, \Cdot]$, we can find $r$ independent columns. 
Let $D^*$ be the components of $D$ corresponding to those $r$ independent columns and $D^{*\perp}$ be the remaining components.
The $r$-by-$r$ submatrix $\hat{M}[E^*, D^*]$, made up of the corresponding independent rows and columns, is invertible. 
Then, by the Implicit Function Theorem, in a neighborhood of $D$ there exists a continuous and differentiable function $g$ such that $D^*=g(D^{*\perp})$.
This identifies $D'$, whose components are $D^*$ and the level set of $g$ corresponding to $D^*$, such that $(H,X)(D')=0$. The level set
defines the finite flexing of the framework. Therefore the system is not rigid.
\qquad\endproof

\begin{remark}
Pinned subspace-incidence frameworks are generalizations of related types of frameworks, 
such as in pin-collinear body-pin frameworks~\cite{jackson2008pin}, 
direction networks~\cite{whiteley1996some},
slider-pinning rigidity~\cite{streinu2010slider},
the molecular conjecture in 2D~\cite{servatius2006molecular},
body-cad constraint system~\cite{haller2012body},
$k$-frames~\cite{white1987algebraic,white1983algebraic}, 
and affine rigidity~\cite{gortler2013affine}.
\end{remark}

\subsection{Required Hypergraph Properties}

\sloppy
This section introduces a pure hypergraph property that will be useful for proving our main theorem.

\fussy
\begin{definition} \label{def:tightness}
A hypergraph $H=(V,E)$ is $(k,0)$-sparse if for any ${V'} \subset V$, 
the induced subgraph ${H'}=({V'},{E'})$ satisfies $|{E'}| \leq k|{V'}|$. 
A hypergraph $H$ is $(k,0)$-tight if $H$ is $(k,0)$-sparse and $|E| = k|V|$.
\end{definition}

This is a special case of the $(k,l)$-sparsity condition 
that was formally studied widely in the geometric constraint solving and combinatorial rigidity literature 
before it was given a name in \cite{lee2007graded}.
A relevant concept from graph matroids is \emph{map-graph}, defined as follows.

\begin{definition}
An \emph{orientation} of a hypergraph is given by identifying as the \emph{tail} of each edge one of its endpoints.
The \emph{out-degree} of a vertex is the number of edges which identify it as the tail and connect $v$ to $V - v$.
A \emph{map-graph} is a hypergraph that admits an orientation such that the out degree of every vertex is
exactly one. 
\end{definition}

The following lemma from \cite{streinu2009sparse} follows Tutte-Nash Williams \cite{tutte1961problem,nash1961edge}
to give a useful characterization of $(k,0)$-tight graphs in terms of maps.

\begin{lemma} 
\label{lem:map_decomposition}
A hypergraph $H$ is composed of $k$ edge-disjoint map-graphs if and only if $H$ is $(k,0)$-tight.
\end{lemma}

\subsection{Proof of Main Theorem and Corollaries}
\label{sec:proof}

We are now ready to prove Theorem~\ref{thm:rigidity_condition}, 
a combinatorial characterization of the existence of finitely many solutions 
for a pinned subspace incidence framework.
The proof adopts an approach by  \cite{white1983algebraic,white1987algebraic}, in proving rigidity of $k$-frames, 
with the following outline:
\begin{itemize}
 \item  We obtain an \emph{expanded mutli-hypergraph} of $\hsxd$ by replacing each hyperedge with $(d-s)$ copies, 
  in order to apply the $(k,0)$-tightness condition.

 \item We show that for a specific form of the rows of a matrix defined on a map-graph, the determinant is not identically zero (Lemma \ref{lem:map_det}). 
 \item We apply Laplace decomposition to the $(d-1,0)$-tight hypergraph as a union of $d-1$ maps, 
  to show that the determinant of the rigidity matrix is not identically zero, as long as a certain polynomial is avoided by the framework (Proof of Main Theorem).
 \item The resulting polynomial is called the \emph{pure condition}
 which characterizes the badly behaved cases
  (i.e.\ the conditions of non-genericity that the framework has to avoid in order for the combinatorial characterization to hold).  
\end{itemize}


First notice that 
the graph property from Theorem \ref{thm:rigidity_condition} is not directly a $(k,0)$-tightness condition, 
so we modify the underlying hypergraph by duplicating each hyperedge into $(d-s)$ copies.

\begin{definition}[Expanded mutli-hypergraph]
Given the underlying hypergraph $H = (\cald,\calsxd)$ of a Pinned Subspace-Incidence problem, 
the \emph{expanded mutli-hypergraph} $\hat{H}=(V,\hat{E})$ of $H$ is obtained by
letting $V = \cald$, and
replacing each hyperedge in $\calsxd$ with $(d-s)$ copies in $\hat{E}$. 
\end{definition}

A rigidity matrix $\hat{M}$ defined in last section for a pinned subspace-incidence framework 
has one row for each hyperedge copy in the expanded multi-hypergraph $\hat{H}=(V,\hat{E})$. 

 Theorem \ref{thm:rigidity_condition} can be restated on the expanded mutli-hypergraph: 

\begin{theorem} \label{thm:rigidity_condition_duplicate}
A pinned subspace-incidence framework is generically minimally rigid if and only if the 
underlying expanded mutli-hypergraph is $(d-1,0)$-tight.
\end{theorem}


Since Theorem \ref{thm:rigidity_condition_duplicate} is equivalent to Theorem \ref{thm:rigidity_condition},
we only need to prove Theorem \ref{thm:rigidity_condition_duplicate} in the following.
%
%
%
We first consider the generic rank of particular matrices defined on a single map-graph. 
\begin{lemma} \label{lem:map_det}
A matrix $N$ defined on a map-graph $H=(V,E)$, such that columns are indexed by the vertices and rows by the edges, 
where the row for hyperedge $x_k \in E$ has non-zero entries only at the $s$ indices corresponding to $v^k_i \in x_k$, 
\begin{equation} \label{eq:map_row_pattern}
\scalemath{0.88}{
[0,\ldots, 0, a_1^k, 0, \ldots, a_2^k, 0, \ldots \ldots, 0, a^k_{s-1}, 0, \ldots,  a^k_s, 0, \ldots, 0]
}
\end{equation}
is generically full rank.
\end{lemma}
%
\begin{proof}
According to the definition of a map-graph, 
the function $\tau:E \rightarrow V$ assigning a tail vertex to each hyperedge  is a one-to-one correspondence.
Without loss of generality, assume that for any $x_k$, 
the corresponding  entry of $\tau(x_k)$  in $N$ is $a^k_s$ 
(notice that we can arbitrarily switch the variable names  $a^k_1, \ldots, a^k_{s-1}, a^k_s$).
The determinant of the matrix $N$ is: 
\begin{equation} \label{eq:map_purecondition}
\det(N) = \pm \prod_k a_s^k +  \sum_{\sigma} \text{sgn}(\sigma) \prod_{i=1}^n N_{i,\sigma_i}
\end{equation}
where $\sigma$ enumerates all permutations of $\{1,2, \ldots, n\}$ except for the one corresponding to the first term $\pm \prod_k a_s^k$. 

Notice that each term $\prod_{i=1}^n N_{i,\sigma_i}$ has at least one $a^k_j, j<s$ as a factor. 
If we use the 
specialization with $a_j^k = 0$ for all $j < s$ and $a_s^k = 1$, 
the summation over $\sigma$  will be zero, 
 and $\det (N)$ will be $\pm \prod_k a_s^k = \pm 1$.
So generically, $N$ must have full rank.
\qquad
\end{proof}

Now we are ready to prove the main theorem 
by  decomposing the expanded mutli-hypergraph as a union of $d-1$ maps, 
and applying Lemma \ref{lem:map_det}. 
%

{\em Proof of Main Theorem}. 
First we show the only if direction. For a generically minimally rigid pinned subspace-incidence framework, the determinant of $\hat{M}$ is not identically zero.
Since the number of columns is $n(d-1)$, it is trivial that $n(d-1)$ copied edges in $\hat{M}$, 
namely $n \frac{d-1}{d-s}$ pins, are necessary. 
It is also trivial to see that $(d-1,0)$-tightness
is necessary, 
since any subgraph ${H'}=({V'},{E'})$ of $\hat{H}$ 
with $|{E'}| > (d-1)|{V'}|$ is overdetermined and generically has no solution.

Next we show the if direction, that $n(d-1)$ edge copies arranged generically in a $(d-1,0)$-tight pattern in the expanded multi-hypergraph imply infinitesimal rigidity.

We first group the columns according to the coordinates. 
In other words, we have $d-1$ groups $C_j$, 
where all columns for the first coordinate belong to $C_1$, 
all columns for the second coordinate belong to $C_2$, etc. 
 This can be done by applying a Laplace expansion to rewrite the determinant of the rigidity matrix $\hat{M}$ 
as a sum of products of determinants (brackets) representing each of the coordinates taken separately:
\[
\det(\hat{M}) = \sum_{\sigma} \left( \pm \prod_j \det \hat{M}[R_j^\sigma, C_j] \right)
\]
where the sum is taken over all partitions $\sigma$ of the rows into $d-1$ subsets
$R^\sigma_1, R^\sigma_2,$ $\ldots,$ $R^\sigma_j,$ $\ldots,$ $R^\sigma_{d-1}$, each of size $|V|$.
Observe that for each $\hat{M}[R^\sigma_{j}, C_{j}]$, 
\[
\det(\hat{M}[R^\sigma_{j}, C_{j}]) = \pm (b^{\sigma 1}      \ldots   b^{\sigma n} ) \det(M'[R^\sigma_{j}, C_{j}])
\]
for some coefficients $ (b^{\sigma 1}    \ldots   b^{\sigma n} )$, 
and each row of $\det(M'[R^\sigma_{j}, C_{j}])$ is either all zero, or of pattern \eqref{eq:map_row_pattern}.
By Lemma \ref{lem:map_decomposition}, the expanded mutli-hypergraph $\hat{H}$ can be decomposed into $(d-1)$ edge-disjoint maps. 
Each such decomposition has some corresponding row partitions $\sigma$, 
where each column group $C_j$ corresponds to a map $N_j$, 
and $R^\sigma_j$ contains rows corresponding to the edges in that map.
%
Observe that $\hat{M}[R^\sigma_j, C_{j}]$ contains an all-zero row $r$, 
if and only if
the 
row $r$ has the $j$th coordinate entry being zero in $\hat{M}$. 
Recall for each hyperedge $x_k$, 
 we are free to pick any $d-s$ rows to include in $\hat{M}$ 
 from the $d-1 \choose s $ rows in the symmetric rigidity matrix $M$. 
 We claim that
 \begin{claim} \label{claim:no_zero_row}
Given a map decomposition, 
we can always pick the rows of the rigidity matrix $\hat{M}$, 
such that there is a corresponding row partition $\sigma^*$,
where none of the minors $\hat{M}[R^{\sigma^*}_j, C_{j}]$ contains an all-zero row.
\end{claim}

For any map $N_j$ in the given map decomposition, for any hyperedge $x_k$,
there are $d-2 \choose s-1$ among the $d-1 \choose s $ rows in $M$ with the $j$th coordinate being non-zero.
Also, it is not hard to show that for all $2 \le s \le d-1$, ${d-2 \choose s-1} \ge d-s$. 
So for any $N_j$ containing $k_j$ copies of a particular hyperedge, 
since all the other maps can pick at most $(d-s) - k_j$ rows from its $d-2 \choose s-1$ choices,
it still has ${d-2 \choose s-1} - ((d-s) - k_j) \ge k_j$ choices. 
Therefore, given a map decomposition, 
we can always pick the rows in the rigidity matrix $\hat{M}$, 
such that there is a partition of each hyperedge's rows,  
where each map $N_j$ get its required rows with non-zeros at coordinate $j$. 
This concludes the proof of the claim.

So by Lemma \ref{lem:map_det}, the determinate of each such minor  $\hat{M}[R^{\sigma^*}_j, C_{j}]$ is generically non-zero. 
We conclude that 
\[
\det(\hat{M}) = \sum_{\sigma} \left( \pm \prod_j  \left( (b^{\sigma 1}   \ldots b^{\sigma n} ) \det \left( M'[R_j^\sigma, C_j]\right) \right) \right)
\]
Observe that each term of the sum has a unique multi-linear coefficient $(b^{\sigma 1}  \ldots b^{\sigma n} )$
that generically do not cancel with any of the others
since $\det(M'[R_j^\sigma,C_j])$ are independent of the $b$'s.
This implies that 
$\hat{M}$ is generically full rank,
thus completes the proof.
Moreover, substituting the values of  $\det(M'[R_j^\sigma,C_j])$ from Lemma \ref{lem:map_det} gives the pure condition for genericity.	
\qquad \endproof

\begin{example}
Consider the pinned subspace-incidence framework  in Example \ref{example:k4} with $d=4,s=2$.
The expanded mutli-hypergraph (replacing each hyperedge with $2$ copies) satisfies $(3,0)$-tightness condition,
and the framework is minimally rigid.
The rigidity matrix $\hat{M}$ has the following form:
\setlength{\extrarowheight}{2pt}
\[
\left[
\scalemath{0.6}{
\begin{array}{ccc|ccc|ccc|ccc}
 b^1_{2} a^1_1 &  -b^1_{1} a^1_1 &  0  &  - b^1_{2} a^1_2 &   b^1_{1} a^1_2 &   0  &  0 &  0 &  0 &  0 &  0 &  0 \\
  b^1_{3} a^1_1 &   0  &   -b^1_{1} a^1_1 &   -b^1_{3} a^1_2 &   0  &   b^1_{1} a^1_2 &  0 &  0 &  0 &  0 &  0 &  0 \\

  0 &  0 &  0 &   b^2_{2} a^2_1 &    -b^2_{1} a^2_1  &   0  &  -b^2_{2} a^2_2&   b^2_{1} a^2_2 &   0  &  0 &  0 &  0 \\
  0 &  0 &  0 &   b^2_{3} a^2_1 &   0  &   -b^2_{1} a^2_1 &   -b^2_{3} a^2_2 &   0  &   b^2_{1} a^2_2 &  0 &  0 &  0 \\

  0 &  0 &  0 &  0 &  0 &  0 &   b^3_{2} a^3_1 &   -b^3_{1} a^3_1 &   0  &    -b^3_{2} a^3_2 &   b^3_{1} a^3_2 &   0  \\
  0 &  0 &  0 &  0 &  0 &  0 &   b^3_{3} a^3_1 &   0  &    -b^3_{1} a^3_1 &   -b^3_{3} a^3_2 &   0  &   b^3_{1} a^3_2 \\

   b^4_{2} a^4_1 &    -b^4_{1} a^4_1 &   0  &   0 &  0 &  0 &  0 &  0 &  0 &   -b^4_{2} a^4_2 &  b^4_{1} a^4_2 &   0  \\
  b^4_{3} a^4_1 &   0  &   -b^4_{1} a^4_1 &   0 &  0 &  0 &  0 &  0 &  0 &    -b^4_{3} a^4_2 &   0  &   b^4_{1} a^4_2 \\

   b^5_{2} a^5_1 &  -b^5_{1} a^5_1 &   0  &   0 &  0 &  0 &   -b^5_{1} a^5_2 &   b^2_{2} a^5_2 &   0  &  0 &   0 &  0\\
  b^5_{3} a^5_1 &  0 &   -b^5_{1} a^5_1 &   0&  0&  0&    -b^5_{1} a^5_2 &  0 &   b^2_{3} a^5_2 &  0&  0&  0\\

  0 &  0 &  0 &   b^6_{2} a^6_1 &   -b^6_{1} a^6_1 &   0  &   0 &  0 &  0 &   -b^6_{2} a^6_2 &   b^6_{1} a^6_2 &   0 \\
  0 &  0 &  0  &  b^6_{3} a^6_1 &   0  &  -b^6_{1} a^6_1 &   0 &  0 &  0&    -b^6_{3} a^6_2 &   0 &  b^6_{1} a^6_2 \\
\end{array}}
 \right]
\] 
%
%
%
%
%

After grouping the coordinates, it becomes
\[
\left[
\scalemath{0.6}{
\begin{array}{cccc|cccc|cccc}
 b^1_{2} a^1_1 &  - b^1_{2} a^1_2 &0&0
 & \color{red} \boldsymbol{-b^1_{1} a^1_1}   & \color{red} \boldsymbol{b^1_{1} a^1_2}& \color{red} \boldsymbol{0}& \color{red} \boldsymbol{0}
 &  0   &   0&0&0
 \\
  b^1_{3} a^1_1 &   -b^1_{3} a^1_2 & 0 &   0  
  &0 &0 &0 &0
  & \color{red} \boldsymbol{-b^1_{1} a^1_1  }& \color{red} \boldsymbol{b^1_{1} a^1_2 }& \color{red} \boldsymbol{0 }& \color{red} \boldsymbol{0  }\\

  0 &   b^2_{2} a^2_1  &  -b^2_{2} a^2_2&0
  &\color{red} \boldsymbol{0    }&\color{red} \boldsymbol{-b^2_{1} a^2_1 }&\color{red} \boldsymbol{b^2_{1} a^2_2 }&\color{red} \boldsymbol{0}
  &  0   &   0 &   0  &  0 
  \\
  \color{red} \boldsymbol{0 }&\color{red} \boldsymbol{b^2_{3} a^2_1}&\color{red} \boldsymbol{-b^2_{3} a^2_2 }& \color{red} \boldsymbol{0}
  &  0 & 0 &0 &0
  &  0 &   -b^2_{1} a^2_1 &   b^2_{1} a^2_2 &0
  \\

  0 &0 &   b^3_{2} a^3_1  &    -b^3_{2} a^3_2
  &\color{red} \boldsymbol{0 }&\color{red} \boldsymbol{0  }&\color{red} \boldsymbol{-b^3_{1} a^3_1  }&\color{red} \boldsymbol{b^3_{1} a^3_2}
  &  0 &0 &   0 &   0
  \\
  0 &  0  &   b^3_{3} a^3_1 &   -b^3_{3} a^3_2
  &  0 &  0    &   0 & 0
  &\color{red} \boldsymbol{0 }&\color{red} \boldsymbol{0  }&\color{red} \boldsymbol{-b^3_{1} a^3_1}&\color{red} \boldsymbol{b^3_{1} a^3_2}
\\

   \color{red} \boldsymbol{b^4_{2} a^4_1  }&\color{red} \boldsymbol{0 }&\color{red} \boldsymbol{0 }&\color{red} \boldsymbol{-b^4_{2} a^4_2}
   &    -b^4_{1} a^4_1     &  0 &  0 &  b^4_{1} a^4_2
   &   0   &  0 &  0 &   0  
   \\
  b^4_{3} a^4_1 &   0 &  0&    -b^4_{3} a^4_2
  &   0  &   0 &  0 &0 
  &\color{red} \boldsymbol{-b^4_{1} a^4_1 }&\color{red} \boldsymbol{0 }&\color{red} \boldsymbol{0 }&\color{red} \boldsymbol{b^4_{1} a^4_2}
\\

   b^5_{2} a^5_1 & 0  &   -b^5_{2} a^5_2 & 0
   &\color{red} \boldsymbol{-b^5_{1} a^5_1 }&\color{red} \boldsymbol{0   }&\color{red} \boldsymbol{b^5_{1} a^5_2 }&\color{red} \boldsymbol{0}
   &   0   & 0 & 0 & 0
   \\
  \color{red} \boldsymbol{b^5_{3} a^5_1 }&\color{red} \boldsymbol{0 }&\color{red} \boldsymbol{-b^5_{3} a^5_2 }&\color{red} \boldsymbol{0}
  &  0 &   0 &  0 &  0
  &  -b^5_{1} a^5_1 &   0 &   b^5_{1} a^5_2 &  0
\\

 \color{red} \boldsymbol{0  }&\color{red} \boldsymbol{b^6_{2} a^6_1 }&\color{red} \boldsymbol{0  }&\color{red} \boldsymbol{-b^6_{2} a^6_2}
  &  0    &   -b^6_{1} a^6_1 &   0  &   b^6_{1} a^6_2
  &  0 &   0  &   0  &   0
 \\
  0   &  b^6_{3} a^6_1  &   0 &    -b^6_{3} a^6_2
  &  0  &   0  &   0 &   0
  &\color{red} \boldsymbol{0   }&\color{red} \boldsymbol{-b^6_{1} a^6_1 }&\color{red} \boldsymbol{0 }&\color{red} \boldsymbol{b^6_{1} a^6_2}
  
\end{array}}
 \right]
\] 
where the boldfaced rows inside each column group corresponding to a map decomposition of the expanded mutli-hypergraph. 
\end{example}

%
%
Theorem \ref{thm:rigidity_condition} gives a pure condition that characterizes the badly behaved cases
(i.e.\ the conditions of non-genericity that breaks the combinatorial
characterization of the infinitesimal rigidity). The pure condition is a function of the $a$'s and $b$'s which can be calculated from the particular realization 
(framework) using Lemma \ref{lem:map_det} and the main theorem. 
Whether it is possible to efficiently test for genericity
from the problem's input (the hypergraph and $x_k$'s) is an open problem.

One particular situation avoided by the pure condition is that there can not be more than one pin on a subspace spanned by the same set dictionary vectors. 
This is important, otherwise 
simple counterexamples to the characterization of the main theorem can be constructed.

\begin{example}
Consider the framework in Figure \ref{fig:s_pins_conterexample} with $d=3,s=2$.
There are $2$ pins on each subspace. 
The expanded mutli-hypergraph of the framework is $(2,0)$-tight. 
However, the framework is obviously not rigid. 
\begin{figure}[hbtp]
\begin{center}
\includegraphics[width=.25\linewidth]{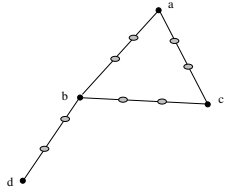}
\end{center}
\caption{An pinned subspace-incidence framework  of $8$ pins and $4$ vertices, with $d=3,s=2$, that violates the pure condition.}
\label{fig:s_pins_conterexample}
\end{figure}
\end{example}

Theorem \ref{thm:rigidity_condition} requires the following genericities: 
\begin{itemize}
 \item The pure condition, which is a function of a given framework.
 \item Generic infinitesimal rigidity, which is the generic rank of the matrix. 
\end{itemize}

The relationship between the two notions of genericities is an open question. Whether  one implies the other is an area of future development. 
However, each of the above conditions applies to an open and dense set. 
Therefore the notion of genericity for the entire theorem that satisfies all of the above conditions
is also open and dense.

\medskip

The combinatorial characterization in Theorem \ref{thm:rigidity_condition}
leads to the proof of Corollary \ref{cor:dictionary_size_bound_generic}, 
which gives a lower bound of dictionary size for generic data points
in the general Dictionary Learning problem. 

{\em Proof of Corollary \ref{cor:dictionary_size_bound_generic} (Lower bound of dictionary size for generic data points)}. 
%
We first prove one direction, that there is generically no dictionary of size $|D|=n$ if $(d-s)m > (d-1)n$.
For any hypothetical $s$-subspace arrangement $\sxd$, the expanded mutli-hypergraph 
$\hat{H}(\sxd)$ - with the given bound for $|D|$ - 
cannot be $(d-1,0)$-sparse.
Hence generically, under the pure conditions of Theorem \ref{thm:rigidity_condition},
 the rigidity matrix - of the $s$-subspace framework $\hsxd$ - with
indeterminates representing the coordinate positions of the points in $D$ -
 has dependent rows. 
 In which case, the original algebraic system $\algesystem$
(whose Jacobian is the rigidity matrix)
will not have a  (complex or real) solution for $D$, with $X$ plugged in.

The converse is implied from our theorem since 
we are guaranteed both generic independence
(the existence of a solution) and generic rigidity (at most finitely many solutions). \qquad \endproof

By characterizing the term ``generically'' in Corollary \ref{cor:dictionary_size_bound_generic}, 
we prove Corollary \ref{cor:dictionary_size_bound}
which gives a lower bound of dictionary size for data points picked uniformly at random from the sphere $S^{d-1}$.

{\em Proof of Corollary \ref{cor:dictionary_size_bound} (Lower bound of dictionary size for highly general data points)}. 
To quantify the term ``generically'' in Corollary \ref{cor:dictionary_size_bound_generic},
we note that the pure-conditions fail only on a measure-zero subset of the
space of frameworks $S_{X,D}$. Since the number of possible
hypergraphs representing the $s$-subspace arrangements is finite for
a given set of pins, it follows that except for a measure-zero subset of
the space of pin-sets $X$, there is no (real or complex) solution to the
algebraic system $\algesystem=0$ when $(d-s)m > (d-1)n$.
Thus when $X$ is picked uniformly at random from the sphere $S^{d-1}$, if $|D|$ is
less than $\big((d-1)/(d-s)\big)|X|$,  with probability 1, there is no solution. \qquad \endproof

\section{Dictionary Learning Algorithm}
\label{sec:algorithm}

In this section, we present the algorithm in Corollary \ref{corr:straight_forward_algo}, 
which constructs  a dictionary of size $n = \left(\dfrac{d-s}{d-1}\right)m$, given $m$ data points picked uniformly at random from the sphere $S^{d-1}$.
The algorithm has two major parts: (1) constructing the underlying hypergraph $H(S_{X,D})$, 
and (2) constructing the $s$-subspace arrangement $\sxd$ and the dictionary $D$. 

\smallskip\noindent \textbf{(1) Algorithm for
constructing the underlying hypergraph $H(S_{X,D})$ for a hypothetical
$s$-subspace arrangement $S_{X,D}$}:

The algorithm  works in three stages to construct a expanded mutli-hypergraph $\hat{H}(S_{X,D})$:
\begin{enumerate}
\item 
We start by constructing a minimal minimally rigid hypergraph $H_0 = (V_0,E_0)$, using the \emph{pebble game algorithm} introduced below.
Here $|V_0| = k(d-s)$, $|E_0| = k(d-1)$, where $k$ is the smallest positive integer such that
$ {k(d-s) \choose s}  \ge k(d-1) $,
so it is possible to construct $E_0$ such that
no more than one hyperedge in $E_0$ containing the same set of vertices in $V_0$.
The values $|V_0|$ and $|E_0|$ are constants when we think of $d$ and $s$ as constants.

\item We use the pebble game algorithm to append a set $V_1$ of $d-s$ vertices and a set $E_1$ of $d-1$ hyperedges 
to $H_0$, such that each hyperedge in $E_1$ contains at least one vertex from $V_1$, 
and the obtained graph $H_1$ is still minimally rigid. 
The subgraph $B_1$ induced by $E_1$ has vertex set $V_{B_1} = V_1  \bigcup V_B$, 
where $V_B \subset V_0$. 
We call the vertex set $V_B$ the \emph{base vertices} of the construction.

\item Each of the following construction step $i$ appends
a set $V_i$ of $d-s$ vertices and a set $E_i$ of $d-1$ hyperedges
such that the subgraph $B_i$ induced by $E_i$ has vertex set $V_i \bigcup V_B$,
and $B_i$ is isomorphic to $B_1$. 
In other words, at each step, 
we directly append a basic structure the same as $(V_1,E_1)$ to the base vertices $V_B$. 
It is not hard to verify that the obtained graph is still minimally rigid.

%
\end{enumerate}
The \emph{pebble game algorithm} by \cite{streinu2009sparse} works on a fixed finite set $V$ of vertices 
and constructs a $(k,l)$-sparse hypergraph. 
Conversely, any $(k,l)$-sparse hypergraph on vertex set $V$ can be constructed by this algorithm.
The algorithm initializes by putting $k$ pebbles on each vertex in $V$. 
There are two types of moves: 
\begin{itemize}
\item \emph{Add-edge}: adds a hyperedge $e$ (vertices in $e$ must contain at least $l+1$ pebbles), removes a pebble from a vertex $v$ in $e$, and assign $v$ as the tail of $e$; 
\item \emph{Pebble-shift}: for a hyperedge $e$ with tail $v_2$, and a vertex $v_1 \in e$ which containing at least one pebble, 
moves one pebble from $v_1$ to $v_2$, and change the tail of $e$ to $v_1$.
\end{itemize}
At the end of the algorithm, if there are exactly $l$ pebbles in the hypergraph, then the hypergraph is $(k,l)$-tight.

Our algorithm runs a slightly modified pebble game algorithm to find a $(d-1,0)$-tight expanded mutli-hypergraph.
We require that each add-edge move adding  $(d-s)$ copies of a hyperedge $e$, 
so a total of $d-s$ pebbles are removed from vertices in $e$.
Additionally, the multiplicity of a hyperedge, not counting the expanded copies, cannot exceed 1. 
For constructing the basic structure of Stage 2, 
the algorithm initializes by putting $d-1$ pebbles on each vertex in $V_1$.
In addition, an add-edge move can only add a hyperedge that contains at least one vertex in $V_1$, 
and a pebble-shift move can only shift a pebble inside $V_1$.

\sloppy
The pebble-game algorithm takes $O\left(s^2 |V_0| {|V_0| \choose s}\right)$ time in  Step 1, 
and $O\left(s^2 \left(|V_0| + (d-s)\right) {|V_0| + (d-s) \choose s }\right)$ time in Step 2. 
Since the entire underlying hypergraph $\hsxd$ has $m = |X|$ edges, 
Step 3 will be iterated $O( m / (d-1))$ times, and each iteration takes constant time. 
Therefore the overall time complexity for constructing $\hsxd$ is 
\[O\left( s^2 \left(|V_0| + (d-s)\right) {|V_0| + (d-s) \choose s } + \left(m/(d-1)\right) \right)\] 
which is $O(m)$ when $d$ and $s$ are regarded as constants. 

\fussy

\smallskip\noindent \textbf{(2) Algorithm for constructing the
$s$-subspace arrangement $S_{X,D}$ and the dictionary $D$}:

The construction of the $s$-subspace arrangement $S_{X,D}$ naturally follows from 
the construction of the underlying hypergraph $\hsxd$.
For the initial hypergraph $H_0$, 
we get a pinned subspace-incidence system $(H_0,X_0)(D_0)$ by arbitrarily choose $|X_0| = |E_0|$ pins from $X$.
Similarly,  for Step 2 and each iteration of Step 3,
we form a pinned subspace-incidence system $(B_i,X_i)(D_i)$ by arbitrarily choosing $|X_i| = d-1$ pins from $X$.

Given $X_0$, we know that the rigidity matrix
 -- of the $s$-subspace framework $H_0(S_{X_0,D_0})$ -- with indeterminates representing
the coordinate positions of the points in $D_0$ -- generically has full rank
(rows are maximally independent), under the pure conditions of Theorem \ref{thm:rigidity_condition};
in which case, the original algebraic subsystem $(H_0,X_0)(D_0)$ 
 (whose Jacobian is the rigidity matrix), with $X_0$ plugged in,
is guaranteed to have a (possibly complex) solution  and only finitely
many solutions for $D_0$. Since the pure conditions fail only on a
measure-zero subset of the space of pin-sets $X_0$, where each pin is in
$S^{d-1}$, it follows that if the pins in $X_0$ are picked uniformly at random from $S^{d-1}$
we know such a solution exists for $D_0$ (and $S_{X_0,D_0}$) and can be found by
solving the algebraic system $H_0(S_{X_0,D_0})$.

Once we have solved $(H_0,X_0)(D_0)$, 
for each following construction step $i$, 
$B_i$ is also rigid since coordintate positions of the vertices in $V_B$ have been fixed
So similarly, we know a solution exists for $D_i$ (and $S_{X_i,D_i}$) and can be found by
solving the algebraic system $B_i(S_{X_i,D_i})$, which is of constant size $O(d)$.
Although there can be more than one choice of solution for each step, 
since every construction step is based on base vertices $V_B$, 
the solution of one  step will not affect any other  steps, 
so generically any choice will result in a successful solution for the entire construction sequence, 
and we obtain $D$ by taking the union of all $D_i$'s.

When we regard $d$ and $s$ as constants, 
the time complexity for Stage (2)
is the constant time for solving the size $O(|V_0|)$ algebraic system $(H_0,X_0)(D_0)$, 
plus $O(m / (d-1))$ times the constant time for solving the size $O(d)$ system $(B_i,X_i)(D_i)$, 
that is $O(m)$ in total. 

Therefore the overall time complexity of the dictionary learning algorithm is $O(m)$. 

\section{Conclusion}
In this paper, we approached Dictionary Learning from a geometric point of view.


We investigated Fitted Dictionary Learning theoretically using machinery from incidence geometry. 
Specifically, 
we formulate the Fitted Dictionary Learning problem as a nonlinear algebraic system $\algesystem$. 
We then follow Asimow and Roth \cite{asimow1978rigidity} to generically linearize $\algesystem$, 
and apply White and Whiteley \cite{white1987algebraic}  to 
obtain a combinatorial rigidity type theorem (our main result) which 
completely characterize the underlying hypergraph $\hsxd$ that are guaranteed to recover a finite number of dictionaries.
%
As corollaries of the main result, 
we gave lower bound for the size of dictionary when the data points are picked uniformly at random, 
and provided an algorithm for Dictionary Learning for such general data points. 
Additionally, we compare several closely related problems of independent interest,
 leading to different directions for future work. 

\bibliographystyle{siam}
\bibliography{refs}






\end{document}